\documentclass[letterpaper, 10 pt, journal, twoside]{ieeetran}  

\IEEEoverridecommandlockouts                              

\usepackage{cite}
\usepackage{amssymb}
\usepackage{amsmath}
\usepackage[dvipdfmx]{graphicx}
\usepackage{color}

\usepackage{algorithm}
\usepackage{algpseudocode}

\newtheorem{dfn}{Definition}

\newtheorem{thm}[dfn]{Theorem}

\newtheorem{prob}[dfn]{Problem}
\newtheorem{rem}[dfn]{Remark}

\newcommand{\bR} { {\mathbb R}}

\algnewcommand{\Initialize}[1]{%
  \State \textbf{Initialize:}
  \State \hspace*{\algorithmicindent}\parbox[t]{0.8\linewidth}{\raggedright #1}
}

\def\red{\hfill $\lhd$}

\title{
Learning Stabilizable Deep Dynamics Models
}

\author{Kenji~Kashima,
        Ryota Yoshiuchi, 
        Yu~Kawano
\thanks{K.~Kashima and R.~Yoshiuchi are with the Graduate School of Informatics,
 Kyoto University, Kyoto, 606-8501,  Japan
  {\tt\small (kk@i.kyoto-u.ac.jp, ryota.yoshiuchi.73r@st.kyoto-u.ac.jp)}}
\thanks{Y.~Kawano is with the Graduate School of Advanced Science and Engineering, 
Hiroshima University, Higashi-Hiroshima, Japan
 {\tt\small (ykawano@hiroshima-u.ac.jp)}}
\thanks{This work was supported by JSPS KAKENHI Grant Number JP21H04875}
}

\begin{document}
\allowdisplaybreaks[4]

\maketitle
\thispagestyle{empty}
\pagestyle{empty}

\begin{abstract}
When neural networks are used to model dynamics, properties
such as stability of the dynamics are generally not guaranteed.
In contrast, there is a recent method for learning the dynamics of
autonomous systems that guarantees global exponential stability
using neural networks. In this paper, we propose a new method
for learning the dynamics of input-affine control systems. An
important feature is that a stabilizing controller and control
Lyapunov function of the learned model are obtained as well.
Moreover, the proposed method can also be applied to solving
Hamilton-Jacobi inequalities. The usefulness of the proposed
method is examined through numerical examples.
\end{abstract}

\begin{IEEEkeywords}
Stabilizable systems, control Lyapunov functions, system identification, deep learning
\end{IEEEkeywords}

\section{Introduction}
Machine learning tools such as neural networks (NNs) are becoming ones of the standard tools for modeling control systems. 
However, as a general problem, system properties such as stability, controllability, and stabilizability are not inherited by learned models.
In other words, there still remains a question: how to implement pre-known systems properties as prior information of learning. 
In this context, there are recent approaches to learning stable \emph{autonomous} dynamics by NNs \cite{chang2020neural,manek2020learning}.
In \cite{chang2020neural}, a term forcing stability has been included in a loss function, and 
states not satisfying a stability condition have been added to learning data in each iteration.
In \cite{manek2020learning}, an NN parametrization of stable system dynamics has been proposed by 
simultaneously modeling system dynamics and a Lyapounov function.
By this approach, it has been explicitly guaranteed that the origin of modeled system dynamics are globally exponentially stable 
for all possible NN parameters.

In this paper, beyond autonomous dynamics, we consider \emph{control} dynamics and provide an NN parametrization of a stabilizable drift vector field when an input vector field is given.
The proposed parametrization theoretically guarantees that the learned drift vector field is stabilizable.
Moreover, as a byproduct of the proposed approach, we can learn not only a drift vector field but also 
a stabilizing controller and a Lyapuonv function of the closed-loop system, i.e., a control Lyapunov function~\cite{sontag1989universal}.
These are utilized to analyze when the true system is stabilized by the learned controller.

The proposed learning method has freedom for tuning parameters, which can be available to solve various control problems 
in addition to learning a stabilizable drift vector field.
This is illustrated by applying our method for solving nonlinear $H_\infty$-control problems and Hamilton-Jacobi inequalities (HJIs). This further suggests a way to modify a loss function for learning in order to solve
Hamilton-Jacobi equations.
In addition, establishing a bridge between our approach and HJIs gives a new look at the conventional method \cite{manek2020learning} for 
stable autonomous dynamics. 
From an inverse optimal control perspective \cite{krstic1998stabilization}, it is possible to show that the learning formula by \cite{manek2020learning} is optimal in some sense. This is also true for our method.

In a related attempt, modeling based on the Hammerstein-Wiener model has been done in \cite{moriyasu2021structured}. However, its internal dynamics have been limited to linear, and the model has not been structurally guaranteed to be stabilizable.
Variances of \cite{chang2020neural,manek2020learning} are found for studying different classes of stable autonomous dynamics 
such as stochastic systems~\cite{lawrence2020almost,urain2020imitationflow}, monotone systems~\cite{wang2020deep}, time-delay systems~\cite{schlaginhaufen2021learning}, and systems admitting positively invariant sets~\cite{takeishi2020learning}.
However, none of them considers control design.
Differently from data-driven control for nonlinear systems, e.g.,  \cite{alsalti2021data,rueda2020data}, a stabilizable drift vector field, stabilizing controller, and Lyapunov function are learned at once, only by specifying them into NNs which have a lot of flexibility to describe nonlinearity.

The remainder of this paper is organized as follows. 
In Section~\ref{pre:sec}, the learning problem of stabilizable unknown drift vector fields is formally stated.
Then, as a preliminary step, we review the conventional work \cite{manek2020learning} for learning stable autonomous systems.
In Section~\ref{main:sec}, as the main result, we present a novel method for simultaneously learning stabilizable dynamics,
a stabilizing controller, and a Lyapuonv function of the closed-loop system,
which are further exemplified in Section~\ref{ex:sec}.
Concluding remarks are given in Section~\ref{con:sec}.

{\it Notation:}
Let $\bR$ and $\bR_+$ be the field of real numbers and the set of non-negative real numbers, respectively.
For a vector or a matrix, $\| \cdot \|$ denotes its Euclidean norm or its induced Euclidean norm, respectively. 
For a continuously differentiable function $V:\bR^n \to \bR$, 
the row vector-valued function consisting of its partial derivatives is denoted by
$(\partial V/\partial x) (x) := [\begin{matrix} \partial V/\partial x_1 & \cdots & \partial V/\partial x_n \end{matrix}](x)$.
Similarly, the column vector-valued function is denoted by $ \nabla V (x) :=(\partial V/\partial x)^\top (x)$.
Furthermore, its Lie derivative along a vector field $f:\bR^n \to \bR^n$ is denoted by 
$L_f V(x) := ((\partial V/\partial x) f) (x)$.
More generally,  $L_g V(x) := ((\partial V/\partial x) g) (x)$ for a matrix-valued function $g:\bR^n \to \bR^{n \times m}$.

\section{Preliminaries}\label{pre:sec}
\subsection{Problem Formulation}
Consider input-affine nonlinear systems, described by
\begin{align}\label{sys}
\dot x = f (x) + g(x) u,
\end{align}
where $f:\bR^n \to \bR^n$ and $g:\bR^n \to \bR^{n \times m}$ are 
locally Lipschitz continuous, and $f(0) = 0$. 

Our goal in this paper is to design a stabilizing controller $u = \alpha (x)$ for 
the system~\eqref{sys} when the drift vector field $f(x)$ is unknown, stated below.
\begin{prob}\label{main:prob}
For the system \eqref{sys}, suppose that 
\begin{itemize}
\item $f$ is unknown, and $g$ is known;
\item  for some input-data $\{x^{(i)}, u^{(i)}\}_{i=1}^{N_d}$, 
the corresponding output data $\{\dot x^{(i)}\}_{i=1}^{N_d}$ are available. 
\end{itemize}
From those available data, learn a stabilizing controller $u = \alpha (x)$ together with 
the drift vector field $f(x)$.
\red
\end{prob}

In Problem \ref{main:prob}, we assume that $\dot x$ is measurable,
which is not an essential requirement. 
If $\{x^{(i)}\}_{i=1}^{N_d}$ is sampled evenly, one can apply a differential approximation method for computing $\dot x^{(i)}$ \cite{chartrand2011numerical}.
In the uneven case, one can utilize the adjoint method  \cite{chen2018neural}.

To guarantee the solvability of the problem,
we suppose that the system~\eqref{sys} is stabilizable in the following sense.
\begin{dfn}\label{stab:dfn}
The system \eqref{sys} is said to be (globally) \emph{stabilizable} if
there exist a scalar-valued function $V: \bR^n \to \bR_+$ and 
a locally Lipschitz continuous function $\alpha: \bR^n \to \bR^m$
such that
\begin{enumerate}
\item $V$ is continuously differentiable;
\item $V$ is positive definite on $\bR^n$, i.e.,  $V(x) \ge 0$ for all $x \in \bR^n$, and  $V(0)=0$ if and only if $x = 0$;
\item $V$ is radially unbounded, i.e., $V(x) \to \infty$ as $\| x \| \to \infty$;
\item it follows that
\begin{align}\label{stab:cond}
L_{f + g \alpha} V(x) < 0
\end{align}
for all $x \in \bR^n \setminus \{ 0\}$.
\red
\end{enumerate}
\end{dfn}

The function $V$ is nothing but a Lyapunov function of the closed-loop system $\dot x = f(x) + g(x) \alpha (x)$, which guarantees 
the global asymptotical stability (GAS) at the origin.
In other words, $V$ is a control Lyapunov function (CLF)~\cite{sontag1989universal}.
If a CLF is found, it is known that one can construct the following Sontag-type stabilizing controller~\cite{sontag1989universal}:
\begin{align}\label{eq:Sontag}
    &u(x) = \alpha_s (x) \\
    &\quad \alpha_s:=\left\{\begin{array}{ll}
    0 & \mbox{if } L_g V = 0 \\
    -\frac{L_f V+\sqrt{\|L_f V\|^2+\|L_g V\|^4}}{\|L_g V\|^2} L_g^\top V
    & \mbox{otherwise}
    \end{array}\right.
    \nonumber
\end{align}
This is one of the well known controllers for nonlinear control and is investigated from the various aspect such as the inverse optimality; see. e.g., \cite{krstic1998stabilization}.

In this paper, we simultaneously learn a drift vector field $f$, CLF $V$, and stabilizing controller $\alpha$.
One can further construct the Sontag-type controller \eqref{eq:Sontag} from the CLF and employ it instead of learned $\alpha (x)$.

\subsection{Learning stable autonomous dynamics}
A neural network (NN) algorithm for learning stable autonomous dynamics has been proposed by \cite{manek2020learning}.
An important feature of this algorithm is that the global exponential stability (GES) of the learned dynamics 
is guaranteed theoretically.
In this subsection, we summarize this algorithm as a preliminary step of solving Problem~\ref{main:prob}.

Consider the following autonomous systems:
\begin{align}\label{asys}
\dot x = f(x).
\end{align}
Suppose that the origin is GES. Then, it is expected that
there exists a Lyapunov function $V: \bR^n \to \bR_+$ satisfying the following three:
\begin{enumerate}
\item $V$ is continuously differentiable;
\item there exist $c_1, c_2 > 0$ such that $c_1 \| x \|^2 \le V(x) \le c_2 \| x \|^2$ for all $x \in \bR^n$;
\item there exists $c_3 > 0$ such that $L_f V (x) \le  - c_3 V(x)$ for all $x \in \bR^n$.
\end{enumerate}
This is true if $\partial f/\partial x$ is continuous and bounded on $\bR^n$; see, e.g., \cite[Theorem 4.14]{khalil2002nonlinear}.

To learn unknown stable dynamics \eqref{asys} by deep learning, we introduce two NNs.
Let $\hat f:= \hat f\mbox{--NN}_{w_{\hat f}, v_{\hat f}, b_{\hat f}} :\bR^n \to \bR^n$ and $V:=V\mbox{--NN}_{w_V, v_V, b_V}:\bR^n \to \bR_+$ 
denote NNs corresponding to a \emph{nominal} drift vector field and Lyapunov function, respectively.
By nominal, we emphasize that $\hat f$ itself does not represent learned stable dynamics, and $f$ is learned as a projection of $\hat f$ onto a set of stable dynamics.

First, we specify the structure of $V$ such that items 1) and 2) hold for arbitrary parameters of the NN.
Define
\begin{align}\label{V}
V(x) :=  \sigma_k ( \gamma (x) - \gamma (0) ) + \varepsilon \| x\|^2,
\end{align}
where $\varepsilon >0$ is given.
The function $\gamma : \bR^n \to \bR$ is an input-convex neural network (ICNN) \cite{amos2017input}, described by
\begin{align} \label{ICNN}
\left\{\begin{array}{l}
z_1 := \sigma_0 (w_0 x + b_0) \\
z_{i+1} := \sigma_i(v_i z_i + w_i x + b_i), \; i=1,\dots,k-1\\
\gamma (x) := z_k,
\end{array}\right.
\end{align}
where $w_i \in \bR^n$, $i = 0,1,\dots, k-1$ and 
$v_j > 0$, $j = 1, \dots, k-1$ represent the weights of the mappings from
$x$ to the $i+1$th layer and from $z_j$ to the $j+1$th layer, respectively, and 
$b_i \in \bR$, $i=0,1,\dots, k-1$ represent the bias functions of the $i$th layer.
Finally, the activate functions $\sigma_i :\bR \to \bR_+$, $i=0,1,\dots, k$ are the following smooth ReLU functions:
\begin{align}\label{sReLU}
\sigma_i (y_i) :=
    \left\{\begin{array}{ll}
    0& \mbox{if } y_i\leq 0 \\
    y_i^2/2 d_i & \mbox{if } 0 < y_i < d_i \\
    y_i - d_i /2 & \mbox{otherwise}
    \end{array}\right.
\end{align}
for some fixed $d_i > 0$, $i = 0,1, \dots, k$.
It has been shown by \cite[Theorem 1]{manek2020learning} that
$V$ constructed by \eqref{V}--\eqref{sReLU} satisfies items 1) and 2) on $\bR^n$ for 
arbitrary parameters $w_i$, $v_i> 0$, $b_i$, and $d_i > 0$.

Next, we consider item 3). 
Let $\hat f$ be locally Lipschitz continuous on $\bR^n$ and satisfy $\hat f (0) = 0$.
One can confirm that the following $f$ satisfies item 3) for arbitrary $\hat f$ and $V$:
\begin{align}\label{learn_sys}
f(x) &:= \hat f(x) + \hat k (x) \\
&\hat k :=
\left\{\begin{array}{ll}
0 & \mbox{if } L_{\hat f} V  \le - c_3 V \\
- \frac{L_{\hat f} V + c_3 V}{\| \nabla V\|^2} \nabla V & \mbox{otherwise}
\end{array}\right..
\nonumber
\end{align}
Since $\nabla V$ is locally Lipschitz continuous, and $\nabla V(x) = 0$ if and only if $x = 0$,
$f$ is locally Lipschitz continuous on $\bR^n \setminus \{0\}$, and $f(0) = 0$;
this has not been explicitly mentioned by \cite[Theorem 1]{manek2020learning}.

Finally, to learn $f$ that fits to data $\{x^{(i)}, \dot x^{(i)}\}_{i=1}^{N_d}$, 
we use the following loss function, for some $\Delta \subset \{1,\dots,N_d \}$,
\begin{align}\label{loss}
L =  \frac{1}{n |\Delta|} \sum_{i\in\Delta} \| \dot x^{(i)} - f (x^{(i)}) \|^2,
\end{align}
where $| \Delta |$ is the cardinality of $\Delta$.
The learning algorithm is summarized in Algorithm \ref{alg:sys}, where
we use the following compact description of \eqref{learn_sys} although the definition at the origin becomes vague:
\begin{align}\label{learn_sys2}
f(x) &:= \hat f(x)  - \frac{\mbox{ReLU}( L_{\hat f} V(x) + c_3 V(x))}{\| \nabla V (x)\|^2} \nabla V(x),
\end{align}
where
\begin{align}
\mbox{ReLU}(y) :=
\left\{\begin{array}{ll}
0 & \mbox{if } y \le 0 \\
y & \mbox{otherwise}
\end{array}\right..
\end{align}

\begin{algorithm}[htbp]
  \caption{Learning Deep Stable Dynamics}
  \label{alg:sys}
  \begin{algorithmic}
    \Require{$c_3>0$, $\varepsilon > 0$, $\theta \in (0, 1)$, $\{x^{(i)}, \dot{x}^{(i)}\}_{i=1}^{N_d}$}
    \Ensure{$f$ and $V$}
    \Initialize{$\hat f:=\hat f\mbox{--NN}_{w_{\hat f}, v_{\hat f}, b_{\hat f}}$\\
    $V:= V\mbox{--NN}_{w_V, v_V, b_V}$\\
    $f:= \hat f - \frac{\mathrm{ReLU}(L_{\hat f} V + c_3 V)}{\| \nabla V\|^2} \nabla V$}
    \Repeat
        \State select $\Delta \subset \{1,\dots,N_d \}$
        \State $L  \leftarrow \frac{1}{n |\Delta |} \sum_{i\in\Delta} \| \dot x^{(i)} - f (x^{(i)}) \|^2$
        \State $w_i \leftarrow w_i + \theta (\partial L/\partial w_i)^\top$, $i=\hat f, V$
        \State $v_i \leftarrow v_i + \theta (\partial L/\partial v_i)^\top$, $i=\hat f, V$
        \State $b_i \leftarrow b_i + \theta (\partial L/\partial b_i)^\top$, $i=\hat f, V$
    \Until convergence
    \State return $f$ and $V$
  \end{algorithmic}
\end{algorithm}

\section{Learning stabilizing controllers}\label{main:sec}
\subsection{Main results}
Inspired by Algorithm \ref{alg:sys}, we present an algorithm for solving Problem~\ref{main:prob}.
Our approach is to learn a drift vector field $f$ and a controller $u = \alpha (x)$ such that 
the GAS of the closed-loop system $\dot x = f(x) + g(x) \alpha (x)$ is guaranteed theoretically.

To this end, we again employ $\hat f:= \hat f\mbox{--NN}_{w_{\hat f}, v_{\hat f}, b_{\hat f}}$ and $V:=V\mbox{--NN}_{w_V, v_V, b_V}$ and 
newly introduce an NN representing a controller, $\alpha := \alpha\mbox{--NN}_{w_\alpha, v_\alpha, b_\alpha}:\bR^n \to \bR^m$.
Recall that $V$ constructed by \eqref{V} -- \eqref{sReLU} satisfies  items 1) -- 3) of Definition \ref{stab:dfn}.
Therefore, the remaining requirement is item 4), which holds for arbitrary $\hat f$, $V$, and $\alpha$ if 
$f$ is learned by
\begin{align}\label{learn_CLF}
f(x) &:= \hat f(x) + \hat \ell (x) \\
&\hat \ell :=
\left\{\begin{array}{ll}
0 & \mbox{if } L_{\hat f + g \alpha } V  \le - W \\
- \frac{L_{\hat f + g \alpha} V + W}{\| \nabla V\|^2} \nabla V & \mbox{otherwise}
\end{array}\right.,
\nonumber
\end{align}
where $W:\bR^n \to \bR_+$ is a given locally Lipschitz continuous positive definite function.
The formula \eqref{learn_CLF} can be viewed as a projection of $\hat f$ onto a \emph{stabilizable} drift vector field for given $g$.
Indeed, $u= \alpha (x)$ stabillizes the learned $f$, stated below.

\begin{thm}\label{thm:CLF}
Consider $V=V\mbox{--NN}_{w_V, v_V, b_V}$ constructed by \eqref{V} -- \eqref{sReLU},
and locally Lipschitz continuous $\hat f= \hat f\mbox{--NN}_{w_{\hat f}, v_{\hat f}, b_{\hat f}}$ and 
$\alpha= \alpha\mbox{--NN}_{w_\alpha, v_\alpha, b_\alpha}$
such that $\hat f(0) = 0$ and $\alpha (0) = 0$.
Also, let $W$ be locally Lipschitz continuous and positive definite.
Then, for $f$ in \eqref{learn_CLF}, the closed-loop system $\dot x = f(x) + g(x) \alpha (x)$ is 
GAS (GES if $W(x) = c_3 V(x)$, $c_3 >0$) at the origin.
\end{thm}

\begin{proof}
As mentioned above, $V$ satisfies  items 1) -- 3) of Definition \ref{stab:dfn}.
By a similar reasoning mentioned in the previous subsection, 
$f + g \alpha$ is locally Lipschitz continuous on $\bR^n \setminus \{0\}$ and satisfies $f(0) + g(0) \alpha (0)= 0$.

Next, it follows from \eqref{learn_CLF} that
\begin{align}\label{Lyap_clsys}
&L_{f + g \alpha} V(x) \nonumber\\
&= L_{\hat f + g \alpha} V(x) + L_{\hat \ell} V(x) \nonumber\\
&=L_{\hat f + g \alpha} V(x) \nonumber\\
&\quad + 
\left\{\begin{array}{ll}
0 & \mbox{if } L_{\hat f + g \alpha} V(x)  \le - W(x) \\
- (L_{\hat f + g \alpha} V(x) + W(x)) & \mbox{otherwise}
\end{array}\right. \nonumber\\
&\le - W(x), 
\quad
\forall x \in \bR^n.
\end{align}
Therefore, the system $\dot x = f(x) + g(x) \alpha (x)$ is GAS at the origin.
Finally, the origin is GES if $W(x) = c_3 V(x) $, $c_3 >0$, 
since there exist $c_1, c_2 >0$ such that $c_1 \| x \|^2 \le V(x) \le c_2 \| x \|^2$ as mentioned above.
\end{proof}

\begin{rem}
If one only requires the GAS of the closed-loop system, 
the activate functions $\sigma_i$, $i=0,1,\dots, k$ are not needed to be smooth ReLU functions \eqref{sReLU}.
Because of the term $\varepsilon \| x\|^2$ in \eqref{V}, $V$ satisfies  items 1) -- 3) of Definition \ref{stab:dfn} if
$\sigma_i$, $i=0,1,\dots, k$ are continuously differentiable, and $\sigma_k$ is positive semi-definite. 
Moreover, $\sigma_i$, $i=0,1,\dots, k-1$ can be selected as vector-valued functions, which is also true in the GES case.
\red
\end{rem}

One notices that the learning formula \eqref{learn_CLF} of $f$ does not depend on $u$.
Therefore, training data of $f$ can be generated from the trajectory of $\dot x = f(x)$.
In other words, we only have to choose  $u^{(i)}=0$, $i=1,\dots, N_d$ in Problem \ref{main:prob}
and to employ the loss function \eqref{loss}.
The proposed learning algorithm is summarized in Algorithm \ref{alg:CLF} below, 
where we again use the following compact description of \eqref{learn_CLF}:
\begin{align}\label{learn_CLF2}
f(x) := \hat f(x)  - \frac{\mbox{ReLU}( L_{\hat f + g \alpha} V(x) + W(x))}{\| \nabla V (x)\|^2} \nabla V(x).
\end{align}

\begin{algorithm}[htbp]
  \caption{Learning Deep Stabilizing Controllers}
  \label{alg:CLF}
  \begin{algorithmic}
    \Require{$g$, $W$, $\varepsilon > 0$, $\theta \in (0, 1)$, $\{x^{(i)}, \dot{x}^{(i)}\}_{i=1}^{N_d}$}
    \Ensure{$f$, $V$, and $\alpha$}
    \Initialize{$\hat f:=\hat f\mbox{--NN}_{w_{\hat f}, v_{\hat f}, b_{\hat f}}$\\
    $V:= V\mbox{--NN}_{w_V, v_V, b_V}$\\
    $\alpha:= \alpha\mbox{--NN}_{w_\alpha, v_\alpha, b_\alpha}$\\
    $f:= \hat f - \frac{\mathrm{ReLU}(L_{\hat f + g \alpha} V + W)}{\| \nabla V\|^2} \nabla V$}
    \Repeat
        \State select $\Delta \subset \{1,\dots,N_d \}$
        \State $L  \leftarrow \frac{1}{n |\Delta |} \sum_{i\in\Delta} \| \dot x^{(i)} - f (x^{(i)}) \|^2$
        \State  $w_i \leftarrow w_i + \theta (\partial L/\partial w_i)^\top$, $i=\hat f, V, \alpha$
        \State $v_i \leftarrow v_i + \theta (\partial L/\partial v_i)^\top$, $i=\hat f, V, \alpha$
        \State $b_i \leftarrow b_i + \theta (\partial L/\partial b_i)^\top$, $i=\hat f, V, \alpha$
    \Until convergence
    \State return $f$, $V$, and $\alpha$
  \end{algorithmic}
\end{algorithm}

At the end of this subsection, 
we take the learning error of $f$ into account. 
Let $f_\mathrm{true}$ denote the true drift vector field, where $f$ is the one learned by Algorithm \ref{alg:CLF}.
As expected, if the learning error $\| f_\mathrm{true} (x)  - f(x) \|$ is small,
then the learned controller $u = \alpha (x)$ stabilizes the true system $\dot x = f_\mathrm{true} (x) + g (x) u$ also, stated below.
\begin{thm}\label{stab:thm}
Let us use the same notations as Theorem \ref{thm:CLF}, and 
let $\Omega_c$ denote the level set of $V(x)$, i.e., 
$\Omega_c:= \{x \in \bR^n : V(x) \le c \}$.
Also, define the following set $D \subset \bR^n$:
\begin{align*}
D := \biggl\{x \in \bR^n:  &\; \| f_\mathrm{true}(x) -  f(x)\| \\ 
&< \frac{W(x)}{2 \varepsilon \| x\| + \sum_{i=0}^{k-1} \prod_{j=i+1}^{k-1} \|v_j\|_2 \| w_i \|_2} \biggr\}.
\end{align*}
If there exists $c > 0$ such that $\Omega_c \subset (D \cup \{0\})$, then $\Omega_c$ is a region of attraction for 
the true closed-loop system $\dot x =  f_\mathrm{true} (x) + g(x) \alpha (x)$.
\end{thm}
\begin{proof}
It follows form Theorem \ref{thm:CLF} that
\begin{align*}
L_{ f_\mathrm{true} + g \alpha } V(x)
&= L_{f + g \alpha } V(x) + L_{f_\mathrm{true} - f} V(x) \\
&\le - W(x) + L_{f_\mathrm{true} - f} V(x) .
\end{align*}
Form the representation \eqref{V} -- \eqref{sReLU} of $V(x)$, it is possible to show
\begin{align*}
\| \nabla V (x) \| \le 2 \varepsilon \| x\| + \sum_{i=0}^{k-1} \prod_{j=i+1}^{k-1} \|v_j\|_2 \| w_i \|_2,
\quad
\forall x \in \bR^n.
\end{align*}
This yields
\begin{align*}
&L_{f_\mathrm{true} - f} V(x) \\
&\le \| f_\mathrm{true}(x) -  f(x) \| 
\left(2 \varepsilon \| x\| +\sum_{i=0}^{k-1} \prod_{j=i+1}^{k-1} \|v_j\|_2 \| w_i \|_2\right).
\end{align*}
Therefore, $L_{ f_\mathrm{true} + g \alpha } V(x) < 0$ on $D$.
The statement of the theorem follows from the fact that $\Omega_c$ is the level set. 
\end{proof}

\subsection{Applications to $H_\infty$-control}
In Algorithm \ref{alg:CLF}, there are freedoms for structures $W$ and $\alpha$.
Utilizing them, one can impose some control performances in addition to the closed-loop stability.
To illustrate this, we apply Algorithm \ref{alg:CLF} to designing an $H_\infty$-controller.

Consider the following system:
\begin{align}\label{plant}
\left\{\begin{array}{l}
\dot x = f(x) + g(x) u + g_d (x) d\\
z = h(x),
\end{array}\right.
\end{align}
where $d \in \bR^p$ and $z\in \bR^q$ denote the disturbance and performance output, respectively.
The functions $g_d:\bR^n \to \bR^{n \times p}$ and $h:\bR^n \to \bR^q$ are locally Lipschitz continuous, and 
$h(0) = 0$. 

We consider designing a feedback controller $u = \alpha (x)$ such that for a given $\gamma > 0$,
the closed-loop system satisfies 
\begin{align}\label{L2gain}
\int_0^\infty \| y(t) \|^2 dt \le \gamma^2 \int_0^\infty \| d(t) \|^2 dt,
\end{align}
when $x(0)=0$.
According to \cite{krstic1998stabilization}, this $H_\infty$-control problem is solvable if there exists 
a continuously differentiable positive definite function $V:\bR^n \to \bR_+$ such that
\begin{align*}
L_{f + g \alpha} V(x)  \le - \| h(x) \|^2 - \frac{4\| L_{g_d} V(x) \|^2}{\gamma^2}.
\end{align*}
Then, from \eqref{Lyap_clsys}, one only has to choose $W(x)$ in Algorithm~\ref{alg:CLF} such that
\begin{align}\label{W_Hinf}
W(x) \ge  \| h(x) \|^2 + \frac{4\| L_{g_d} V(x) \|^2}{\gamma^2}.
\end{align}
If $W(x)$ is positive definite, the closed-loop stability is also guaranteed when $d=0$. 

\subsection{Applications to Hamilton-Jacobi inequalities}
As a byproduct of Algorithm \ref{alg:CLF}, a Lyapunov function $V(x)$ of 
the closed-loop system $\dot x = f(x) + g(x) \alpha (x)$ is also learned.
We apply this fact for solving the following Hamilton-Jacobi inequality (HJI):
\begin{align}\label{HJI}
H(x) &:= L_f V(x) - \frac{1}{2} L_g V(x) R^{-1}(x) L_g^\top V(x) + W(x) \nonumber\\
& \le 0,
\quad
\forall x \in \bR^n
\end{align}
with respect to $V:\bR^n \to \bR_+$ for given $W:\bR^n \to \bR_+$ and $R:\bR^n \to \bR^{m \times m}$,
where $R$ is symmetric and positive definite for all $x \in \bR^n$.

One can solve the HJI by specifying the structure of~$\alpha$ in Algorithm~\ref{alg:CLF} into
\begin{align}\label{opt_alpha}
\alpha (x) := - \frac{1}{2} R^{-1}(x) L_g^\top V(x).
\end{align}
Indeed, it follows from \eqref{Lyap_clsys} and \eqref{opt_alpha} that
\begin{align*}
&L_f V(x) - \frac{1}{2} L_g V(x) R^{-1}(x) L_g^\top V(x) \\
&= L_{f + g \alpha} V(x) \le - W(x).
\end{align*}

From the above, one may also notice that the Hamilton-Jacobi equation (HJE), $H(x) = 0$,
can be solved approximately by making $\| H (x) \|$ small.
This, for instance, can be done by replacing the loss function~\eqref{loss} with
\begin{align}\label{opt_loss}
L =  \frac{1}{n |\Delta|} \sum_{i\in\Delta} & \left(  \| \dot x^{(i)} - f (x^{(i)}) \|^2  + a \| H (x^{(i)}) \|^2 \right),
\end{align}
where $a > 0$ is the weight.
From standard arguments of optimal control \cite{krstic1998stabilization}, this further implies that
$u = 2 \alpha (x)$ for $\alpha$ in \eqref{opt_alpha} is an approximate solution to
the following optimal control problem:
\begin{align}\label{opt_con}
&\inf_{u} \int_0^\infty W(x(t)) + \frac{1}{2} u^\top (t) R(x(t)) u(t) dt\\
&\mbox{subject to } \dot x = f (x)+ g(x) u.
\nonumber
\end{align}
That is, Algorithm \ref{alg:CLF} can also be employed for solving the optimal control problem \eqref{opt_con} approximately.

\subsection{Revisiting learning stable autonomous dynamics}
In the previous subsection, we have established a bridge between Algorithm~\ref{alg:CLF} and optimal control. 
In fact, an optimal control perspective gives 
a new look at the formula \eqref{learn_sys} for learning stable autonomous dynamics.

Inspired by inverse optimal control \cite{krstic1998stabilization}, we relate the formula \eqref{learn_sys} with an HJI \eqref{HJI}.

\begin{thm}\label{invopt:thm}
For arbitrary $\hat f$ of locally Lipschitz on $\bR^n \setminus \{0\}$ and $V$ of class $C^1$, it follows that
\begin{align*}
L_{\hat f} V(x) - \frac{ \| \nabla V(x) \|^2}{2r(x)} \le - c_3 V(x),
\quad
\forall x \in \bR^n,
\end{align*}
for 
\begin{align}\label{invopt_r}
r:=
\left\{\begin{array}{ll}
b & \mbox{if } L_{\hat f} V  \le - c_3 V \\
\frac{\| \nabla V\|^2 }{2(L_{\hat f} V + c_3 V)}  & \mbox{otherwise}
\end{array}\right.
\end{align}
where $b>0$ is arbitrary.
Moreover, $\hat k$ in \eqref{learn_sys} satisfies $\hat k = - \nabla V/(2 r)$ when $b \to \infty$.
\end{thm}
\begin{proof}
It follows from \eqref{invopt_r} that
\begin{align*}
&L_{\hat f} V(x) - \frac{ \| \nabla V(x) \|^2}{2r(x)} \\
&= L_{\hat f} V(x) 
 - \left\{\begin{array}{ll}
\|  \nabla V(x)\|^2/(2 b) & \mbox{if } L_{\hat f} V(x)  \le - c_3 V(x) \\
L_{\hat f} V(x) + c_3 V(x)  & \mbox{otherwise}
\end{array}\right.\\
&\le - c_3 V(x).
\end{align*}
Next, it holds that
\begin{align*}
- \frac{\nabla V(x)}{2 r(x)}
= \left\{\begin{array}{ll}
\frac{\nabla V(x)}{2 b} & \mbox{if } L_{\hat f} V(x)  \le - c_3 V(x) \\
\frac{L_{\hat f} V(x) + c_3 V(x)}{\| \nabla V(x)\|^2}  & \mbox{otherwise}
\end{array}\right.
\end{align*}
Thus, we have $\hat k = - \nabla V/(2 r)$ when $b \to \infty$.
\end{proof}

The above theorem and discussion in the previous subsection imply that 
when $b \to \infty$, the controller $u^* = 2 \hat k$ with $\hat k$ in \eqref{learn_sys} is an optimal controller of
\begin{align*}
&\inf_{u} \int_0^\infty q(x(t)) + \frac{\| u(t) \|^2}{2 r(x(t))}  dt\\
&\mbox{subject to } \dot x = \hat f (x) +  u
\end{align*}
for $r$ in \eqref{invopt_r} and $q$ defined by
\begin{align*}
q(x) := - L_f V(x) + \frac{\| \nabla V(x) \|^2}{2 r(x) } \ge W(x).
\end{align*}
That is, the learning formula \eqref{learn_sys} is optimal in this sense.
A similar remark holds for the learning formula \eqref{learn_CLF} of 
stabilizing control design.

\section{Examples}\label{ex:sec}
In this section, we illustrate Algorithm \ref{alg:CLF}.
As system dynamics, we consider the following van der Pol oscillator:
\begin{align}\label{vdP}
\dot x &= f_\mathrm{true} (x) + g u\\
&f_\mathrm{true} (x) 
:= 
\begin{bmatrix}
x_2\\
 -x_1 + 0.3 (1 -  x_2^2) x_2 
 \end{bmatrix}, \;
g 
= 
\begin{bmatrix}
0 \\ 1
\end{bmatrix}.
\nonumber
\end{align}
When $u = 0$, this system has the stable limit cycle, and thus the origin is unstable.
If the drift vector field~$f_\mathrm{true}$ is known, this system is stabilizable by specifying $\dot x_2$.

For stabilizing control design, training data points $\{x^{(i)}\}_{i=1}^{N_d}$ are equally distributed on $[-3, 3] \times [-3, 3]$,
and the number of training data is $N_d = 10000$.
We choose $\varepsilon$ of the Lyapunov candidate~\eqref{V} as $\varepsilon = 10$, and $W(x)$ in \eqref{learn_CLF} as $W(x) = 500 \| x\|^2$.
A parameter $\theta$ in Algorithm~\ref{alg:CLF} is selected as $\theta = 0.005$.
For optimization, the adaptive moment estimation (Adam) is employed.

Figure~\ref{im:f} shows the phase portrait of the learned dynamics $\dot x = f(x)$.
As shown in Fig.~\ref{im:LC}, the learned dynamics have a stable limit cycle.
Thus, Algorithm~\ref{alg:CLF} preserves a topological property of the true dynamics.
Also, we plot the learned Lyapunov function $V(x)$ and controller $u=\alpha(x)$ in Figs.~\ref{im:V} and~\ref{im:alpha}, respectively.
It can be confirmed that $V(x)$ is positive definite.
This and $W(x) = 500 \| x\|^2$ imply that the learned controller $u=\alpha(x)$ is a stabilizing controller
for the learned drift vector field $f(x)$.
Since $V(x)$ is a CLF,  the Sontag-type controller~\eqref{eq:Sontag} can also be constructed, which is plotted in Fig.~\ref{im:Sontag}.

As confirmed by Fig.~\ref{im:f_cl}, the learned controller $u=\alpha(x)$ stabilizes the learned dynamics $\dot x = f(x) + g u$.
We also apply this controller to the true dynamics $\dot x =  f_\mathrm{true}(x) + g u$.
According to Fig.~\ref{im:ftrue_cl}, the true system is also stabilized. 
However, the conditions in Theorem~\ref{stab:thm} do not hold at $13$ of $10000$ data points.

\begin{figure}[tb]
    \begin{minipage}{0.45\columnwidth}
    \centering
    \includegraphics[width=42mm]{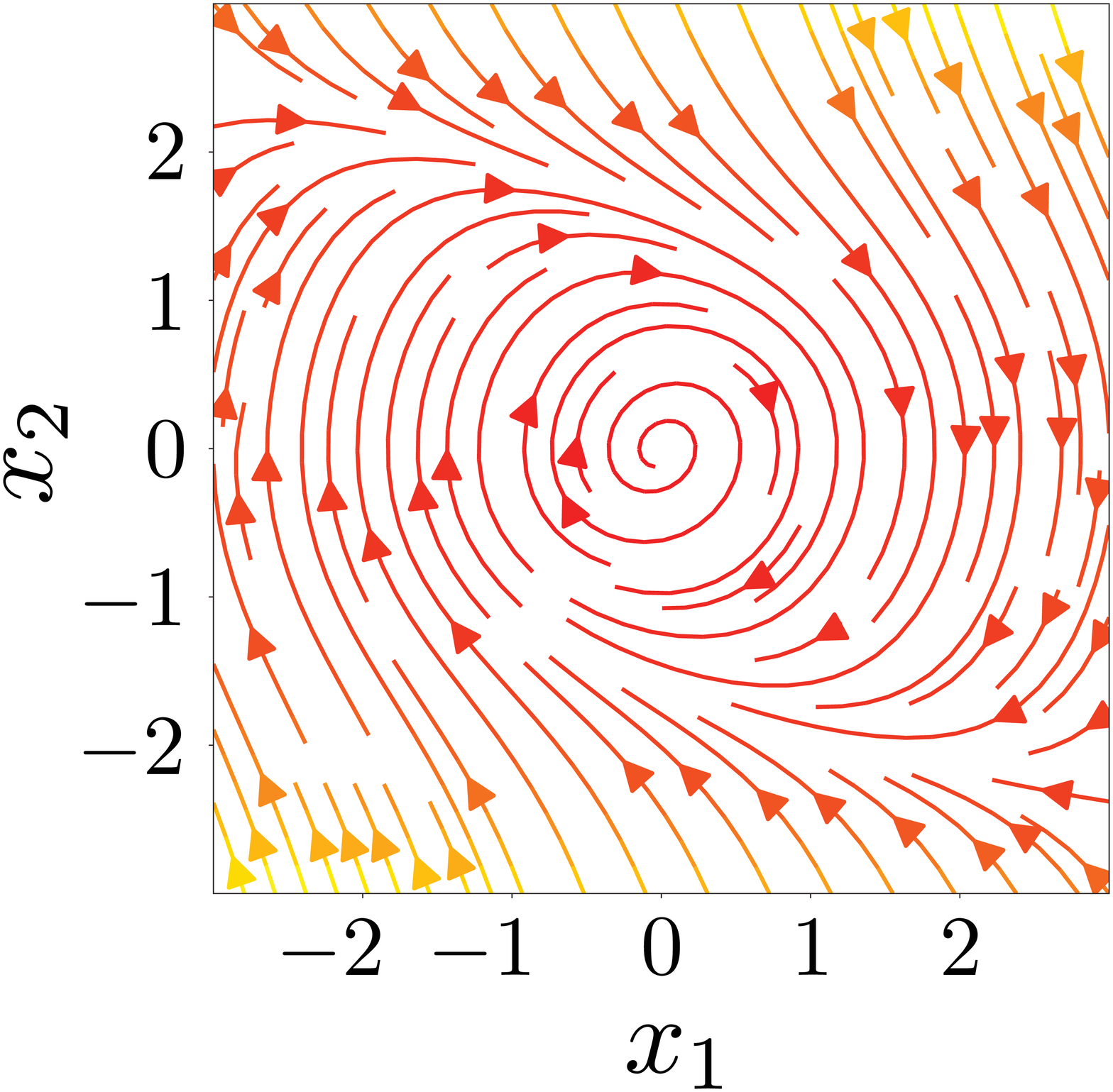}
    \caption{Phase portrait of $\dot x = f(x)$ learned by Algorithm \ref{alg:CLF}}
    \label{im:f}
    \end{minipage}
    \hspace{0.04\columnwidth} 
    \begin{minipage}{0.45\columnwidth}
    \centering
    \includegraphics[width=42mm]{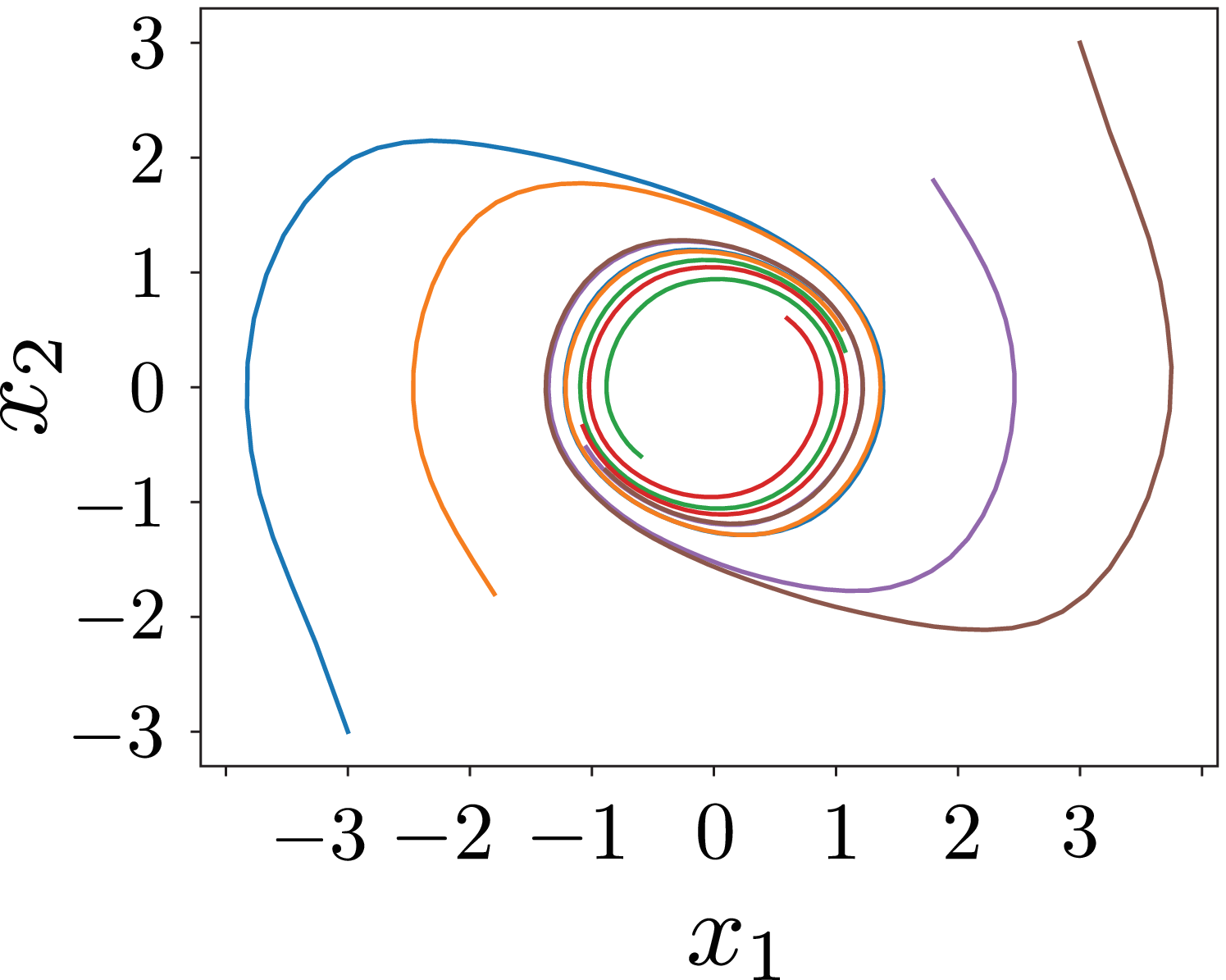}
    \caption{Limit cycle of $\dot x = f(x)$ learned by Algorithm \ref{alg:CLF}}
    \label{im:LC}
    \end{minipage}
\end{figure}

\begin{figure}[tb]
    \begin{minipage}{0.45\columnwidth}
    \centering
    \includegraphics[width=44mm]{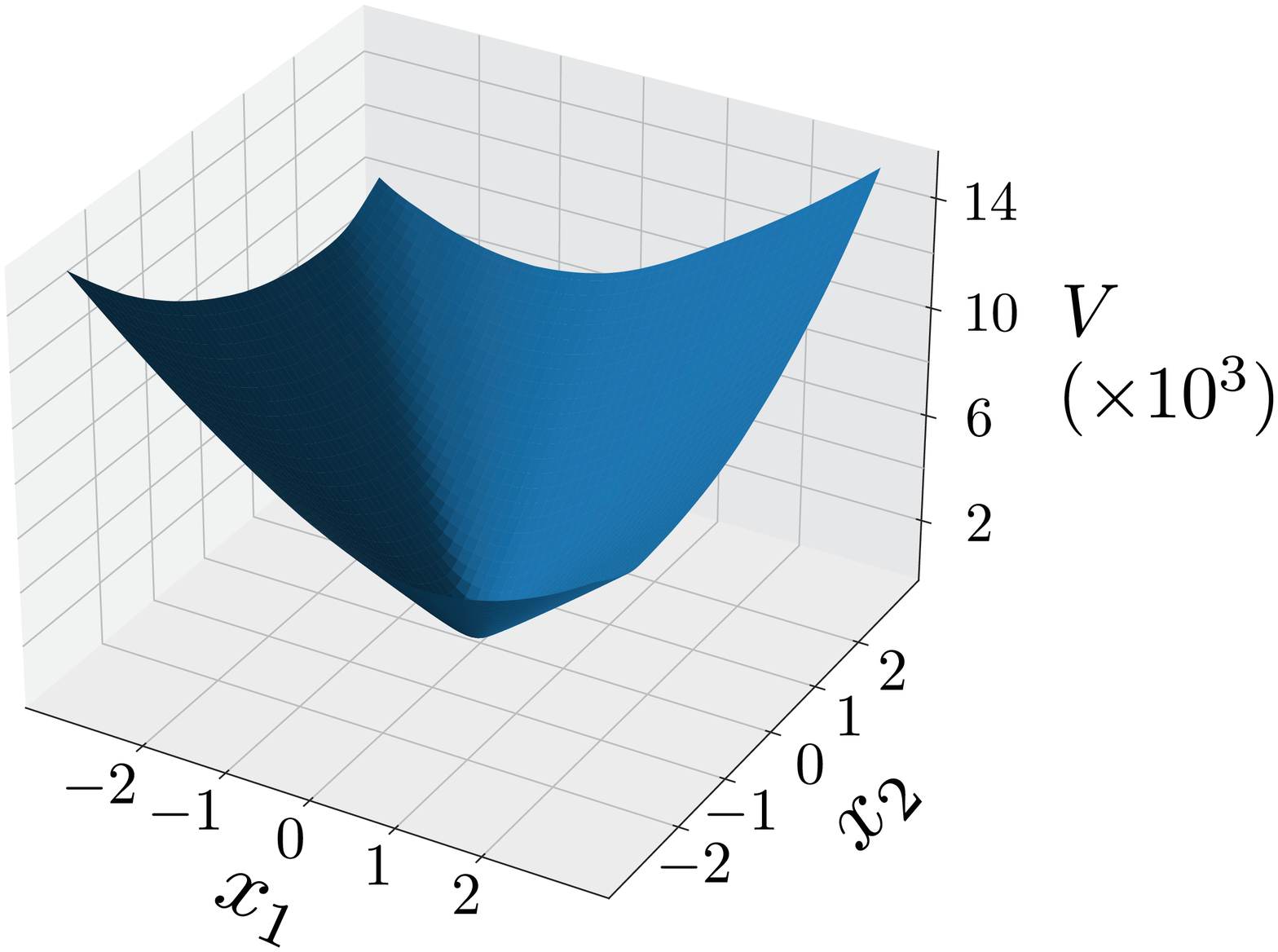}
    \caption{Lyapunov function $V(x)$ learned by Algorithm \ref{alg:CLF}}
    \label{im:V}
    \end{minipage}
    \hspace{0.04\columnwidth} 
    \begin{minipage}{0.45\columnwidth}
    \centering
    \includegraphics[width=38mm]{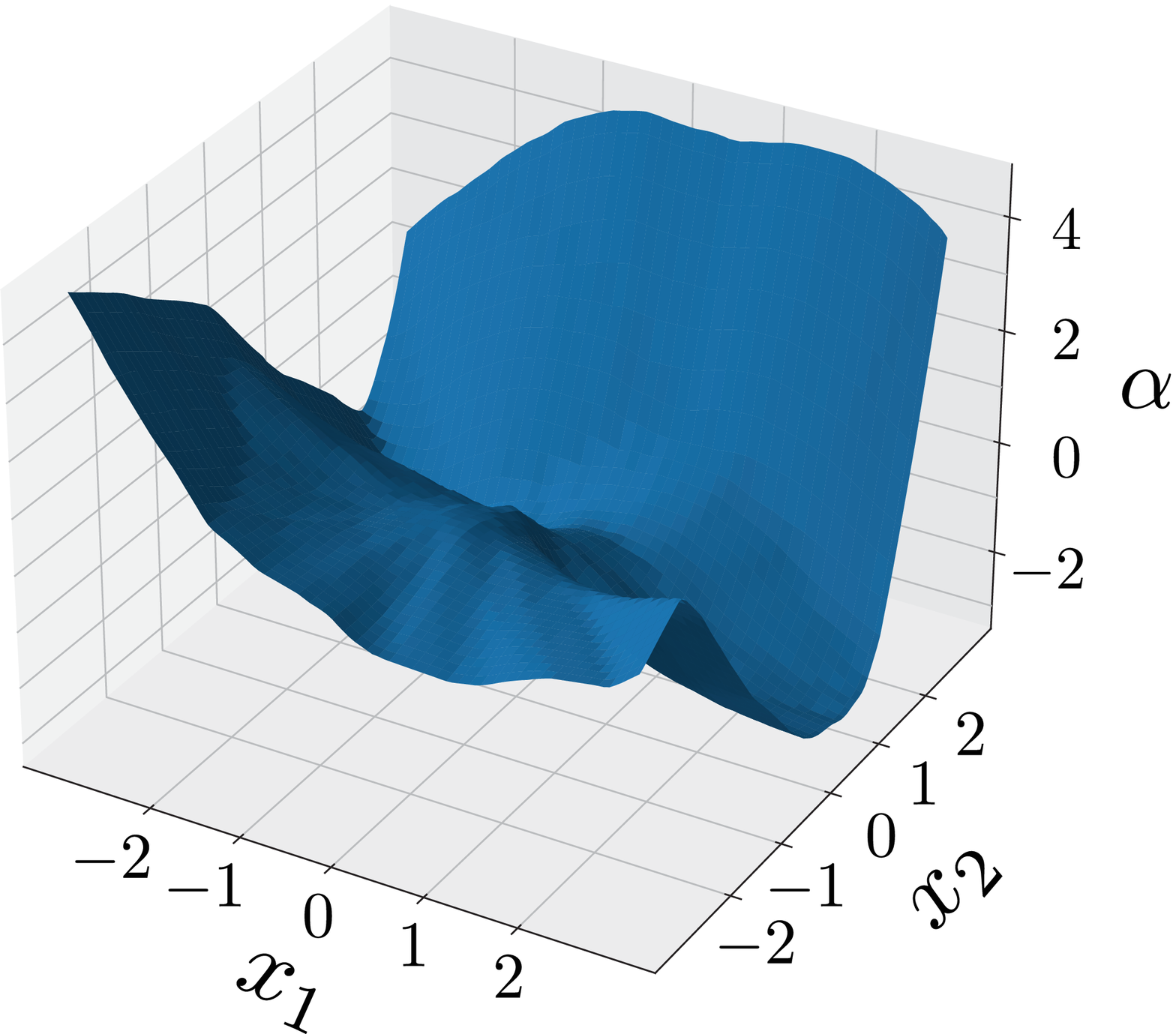}
    \caption{Stabilizing controller $u=\alpha(x)$ learned by Algorithm \ref{alg:CLF}}
    \label{im:alpha}
    \end{minipage}
\\[5mm]
    \centering
    \includegraphics[width=45mm]{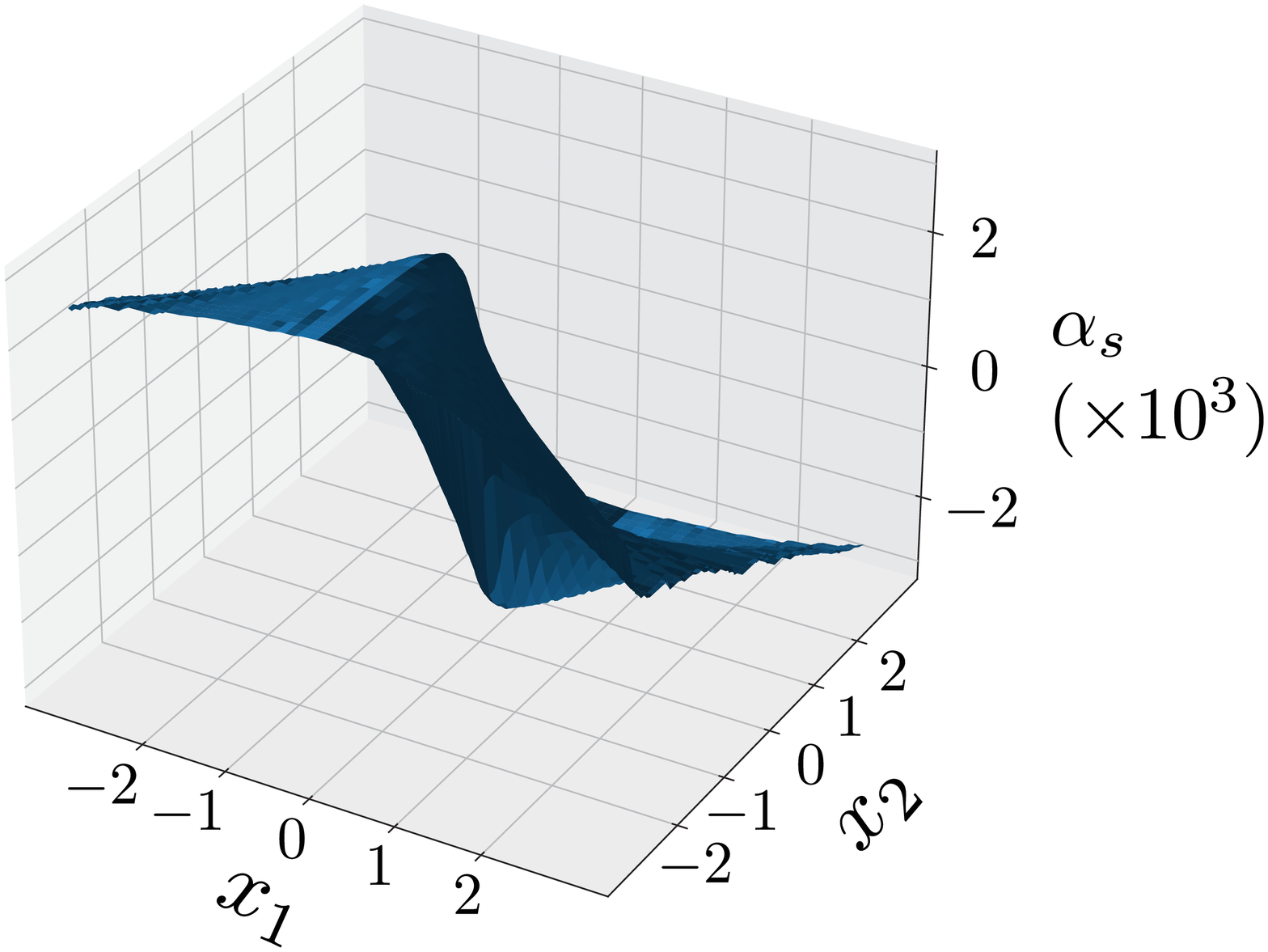}
    \caption{Sontag-type controller for $V(x)$ learned by Algorithm \ref{alg:CLF}}
    \label{im:Sontag}
  \end{figure}

\begin{figure}[tb]
    \begin{minipage}{0.45\columnwidth}
    \centering
    \includegraphics[width=42mm]{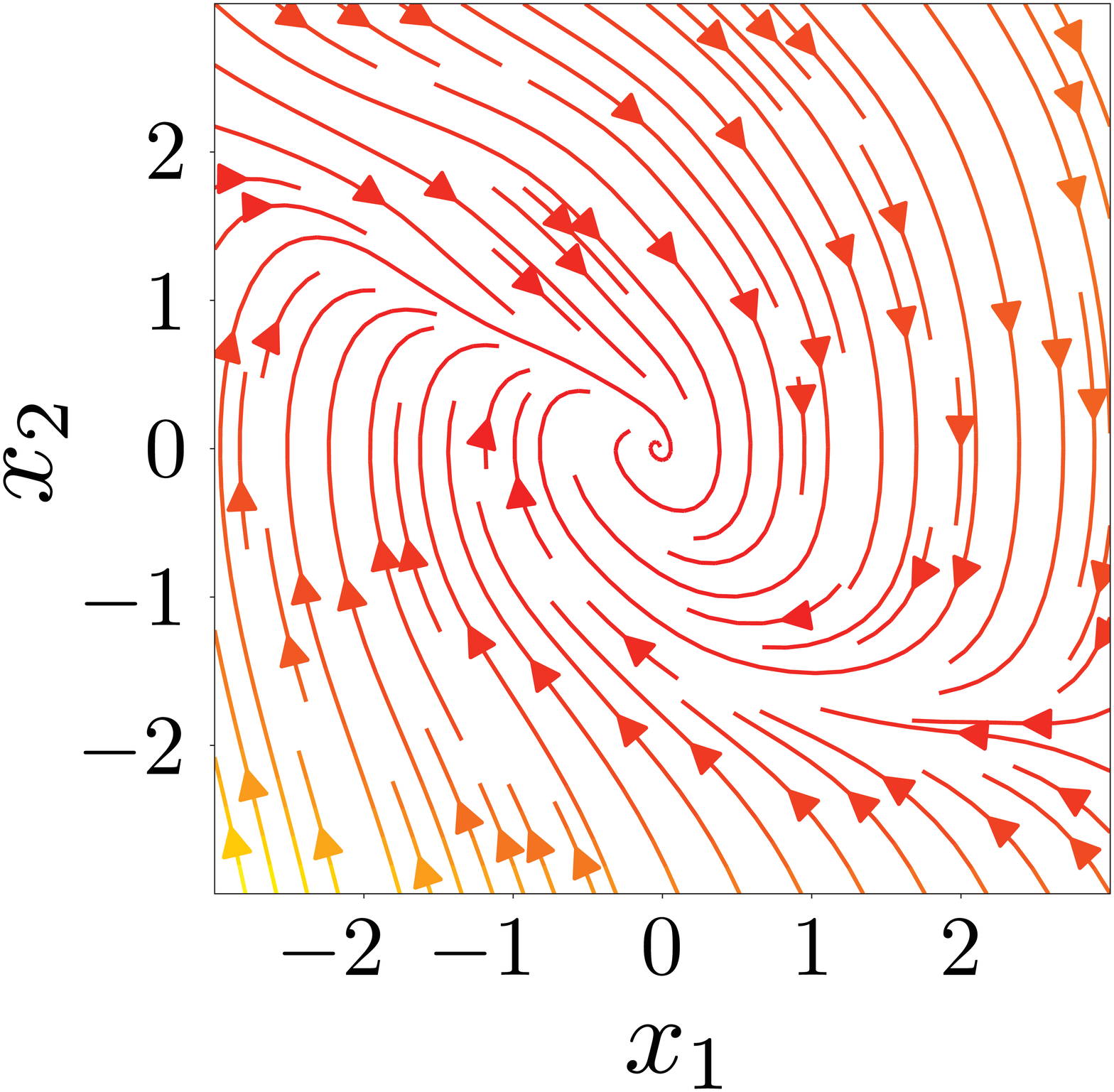}
    \caption{Phase portrait of the closed-loop system: $\dot x = f(x) + g \alpha (x)$}
    \label{im:f_cl}
    \end{minipage}
    \hspace{0.04\columnwidth} 
    \begin{minipage}{0.45\columnwidth}
    \centering
    \includegraphics[width=42mm]{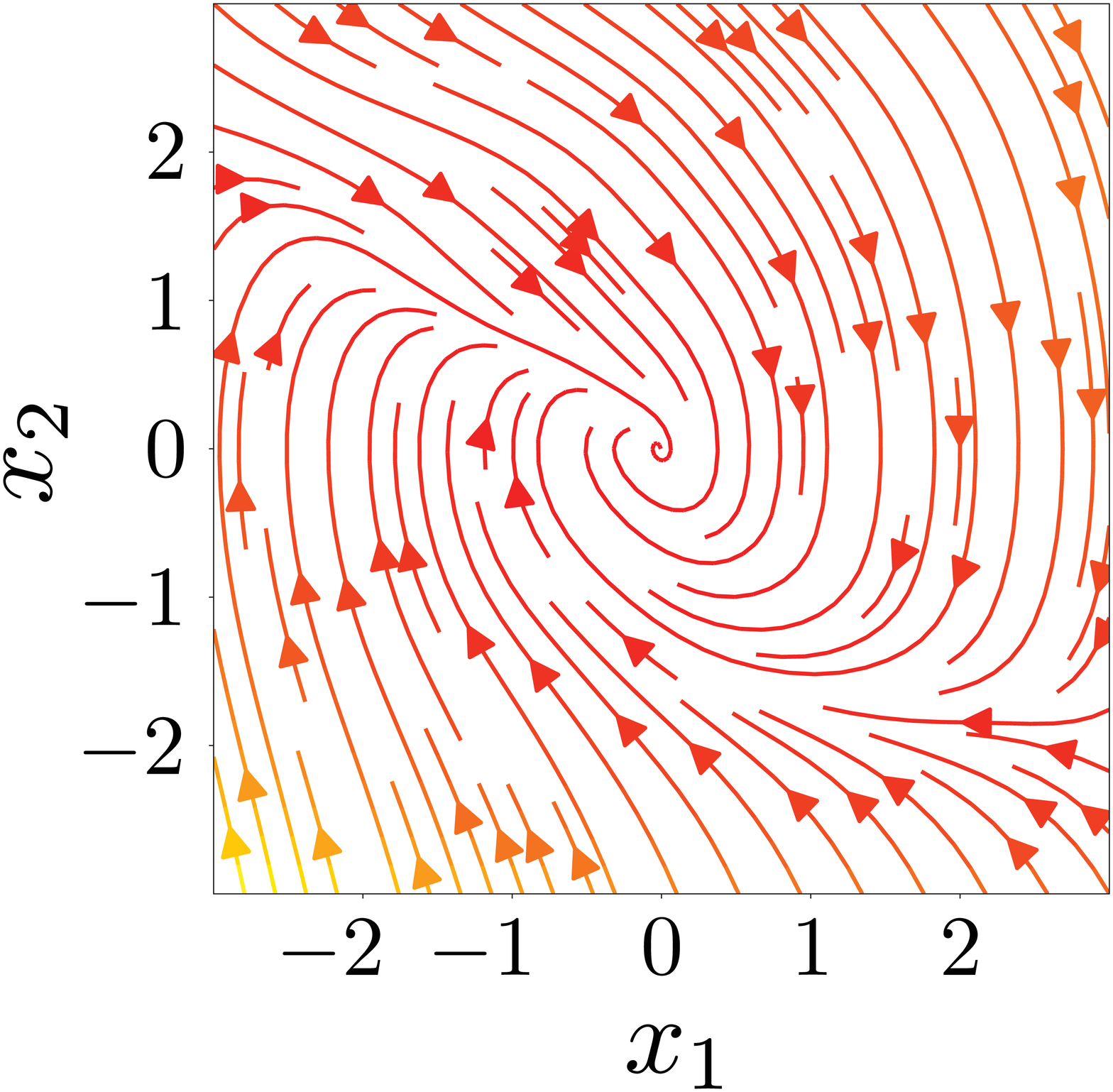}
    \caption{Phase portrait of the true closed-loop system: $\dot x = f_\mathrm{true}(x) + g \alpha (x)$}
    \label{im:ftrue_cl}
    \end{minipage}
\end{figure}

To make the conditions in Theorem~\ref{stab:thm} hold, we select the parameters of Algorithm~\ref{alg:CLF}
as $\varepsilon = 10^{-3}$ and $W(x) = 10^3 \| x\|^2$.
In this case, the learned $f(x)$, $V(x)$, and $\alpha (x)$ by Algorithm~\ref{alg:CLF} satisfy the conditions.
Thus, it is guaranteed that $u= \alpha (x)$ is a stabilizing controller of the true dynamics $\dot x =  f_\mathrm{true}(x) + g u$. 

We plot the phase portrait of the learned dynamics $\dot x = f(x)$ in Fig.~\ref{im:sf} which looks similar to Fig.~\ref{im:f}.
Also, the learned dynamics again preserve the limit cycle as confirmed by Fig.~\ref{im:sLC}.
Therefore, the learned drift vector fields are not sensitive with respect to the parameters of Algorithm~\ref{alg:CLF} at least in this example.
Next, we plot the learned Lyapunov function $V(x)$ and controller $u=\alpha(x)$ in Figs.~\ref{im:sV} and~\ref{im:salpha}, respectively.
The new $\alpha (x)$ in Fig.~\ref{im:salpha} is much larger than the previous one in Fig.~\ref{im:alpha}.
According to Figs.~\ref{im:V} and~\ref{im:sV}, the reason can be
to increase the convergence speed of the $x_2$-direction, 
which can be due to the conservativeness of Theorem~\ref{stab:thm}.
Future work includes to derive a less conservative condition for the closed-loop stability of the true dynamics.

\begin{figure}[b]
    \begin{minipage}{0.45\columnwidth}
    \centering
    \includegraphics[width=42mm]{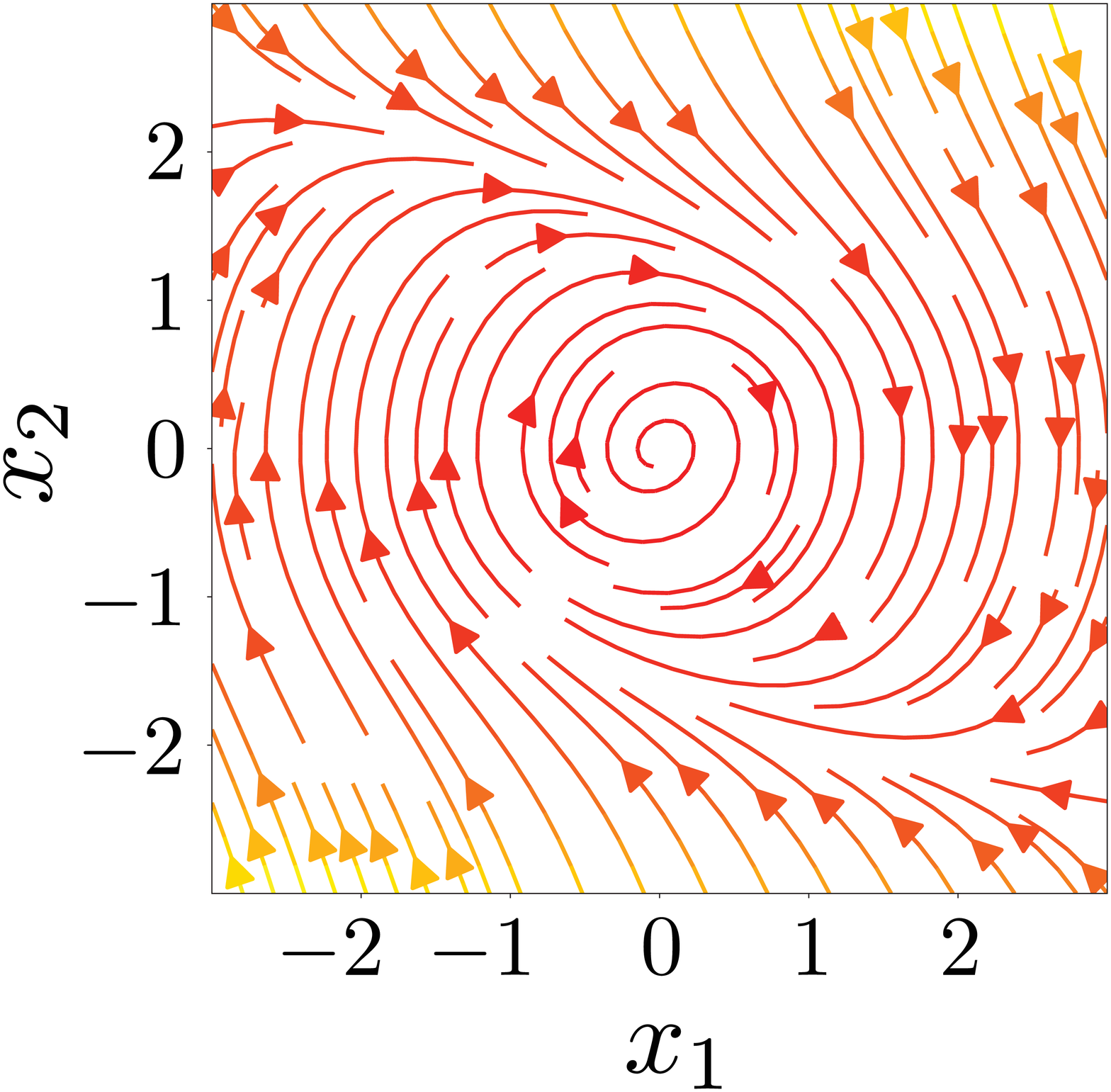}
    \caption{Phase portrait of  $\dot x = f(x)$ learned by Algorithm \ref{alg:CLF} such that the conditions in Theorem \ref{stab:thm} hold}
    \label{im:sf}
    \end{minipage}
    \hspace{0.04\columnwidth} 
    \begin{minipage}{0.45\columnwidth}
    \centering
    \includegraphics[width=42mm]{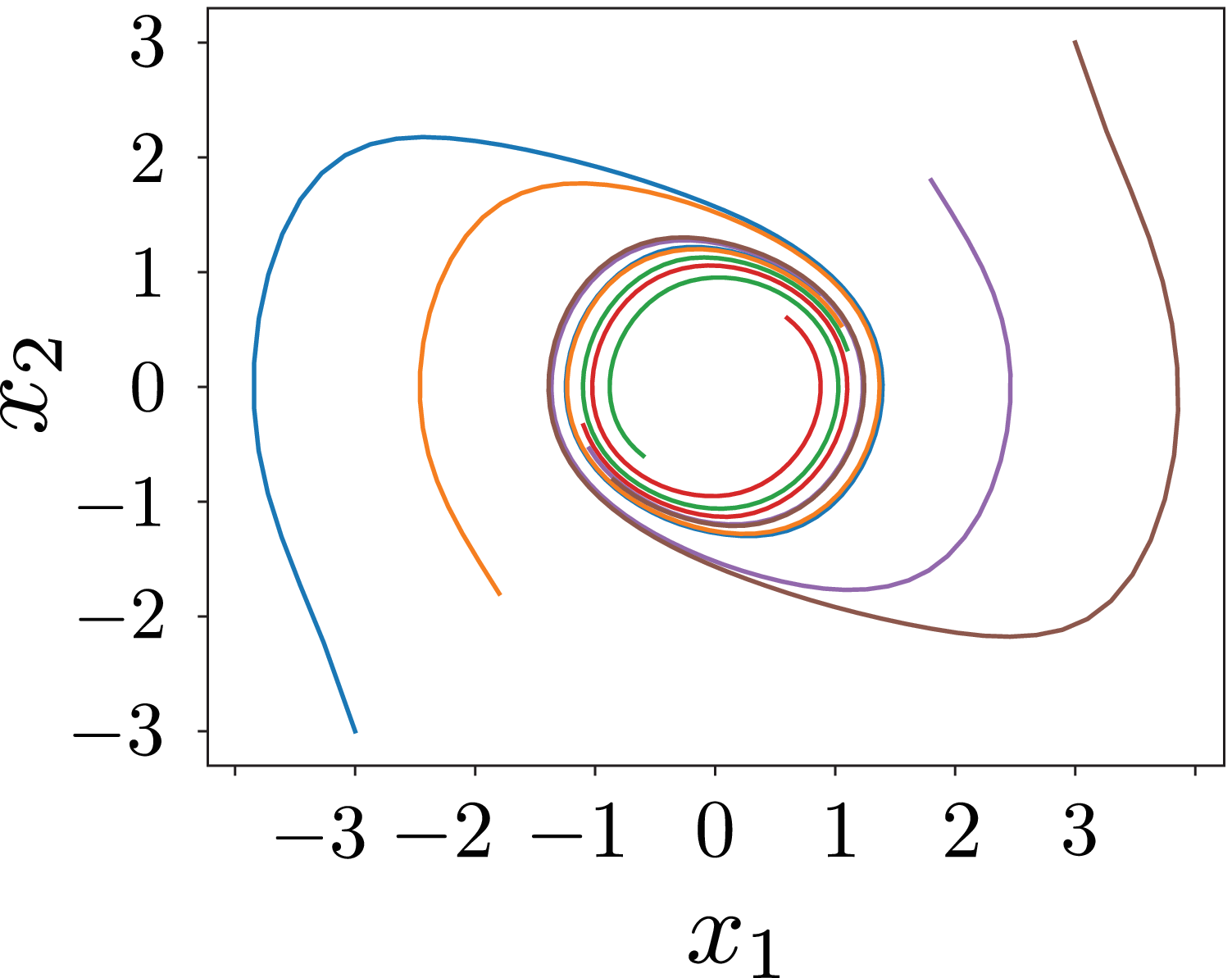}
    \caption{Limit cycle of $\dot x = f(x)$ learned by Algorithm \ref{alg:CLF} such that the conditions in Theorem \ref{stab:thm} hold}
    \label{im:sLC}
    \end{minipage}
\end{figure}

\begin{figure}[tb]
    \begin{minipage}{0.45\columnwidth}
    \centering
    \includegraphics[width=42mm]{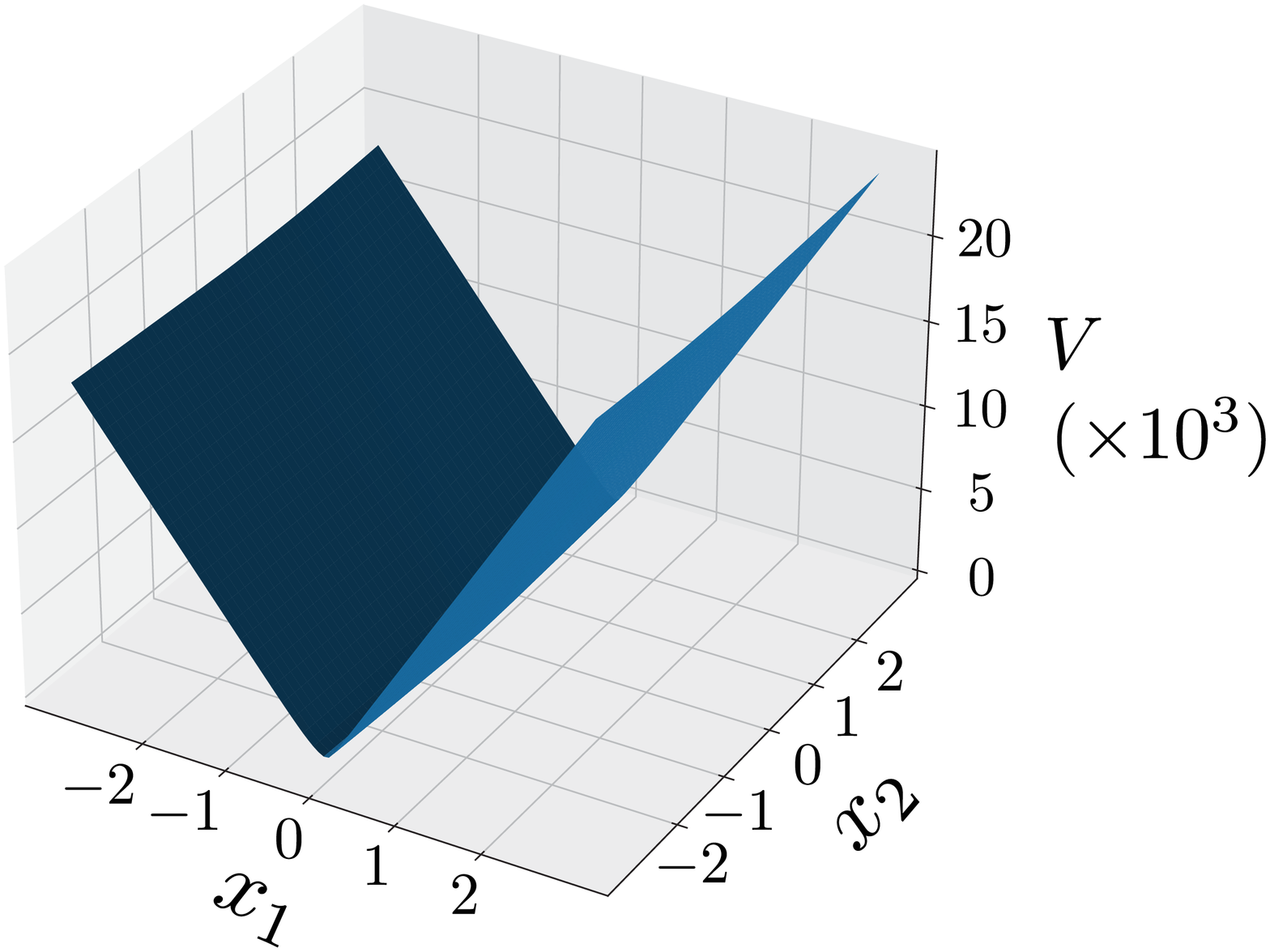}
    \caption{Lyapunov function $V(x)$ learned by Algorithm \ref{alg:CLF} such that the conditions in Theorem \ref{stab:thm} hold}
    \label{im:sV}
    \end{minipage}
    \hspace{0.04\columnwidth} 
    \begin{minipage}{0.45\columnwidth}
    \centering
    \includegraphics[width=42mm]{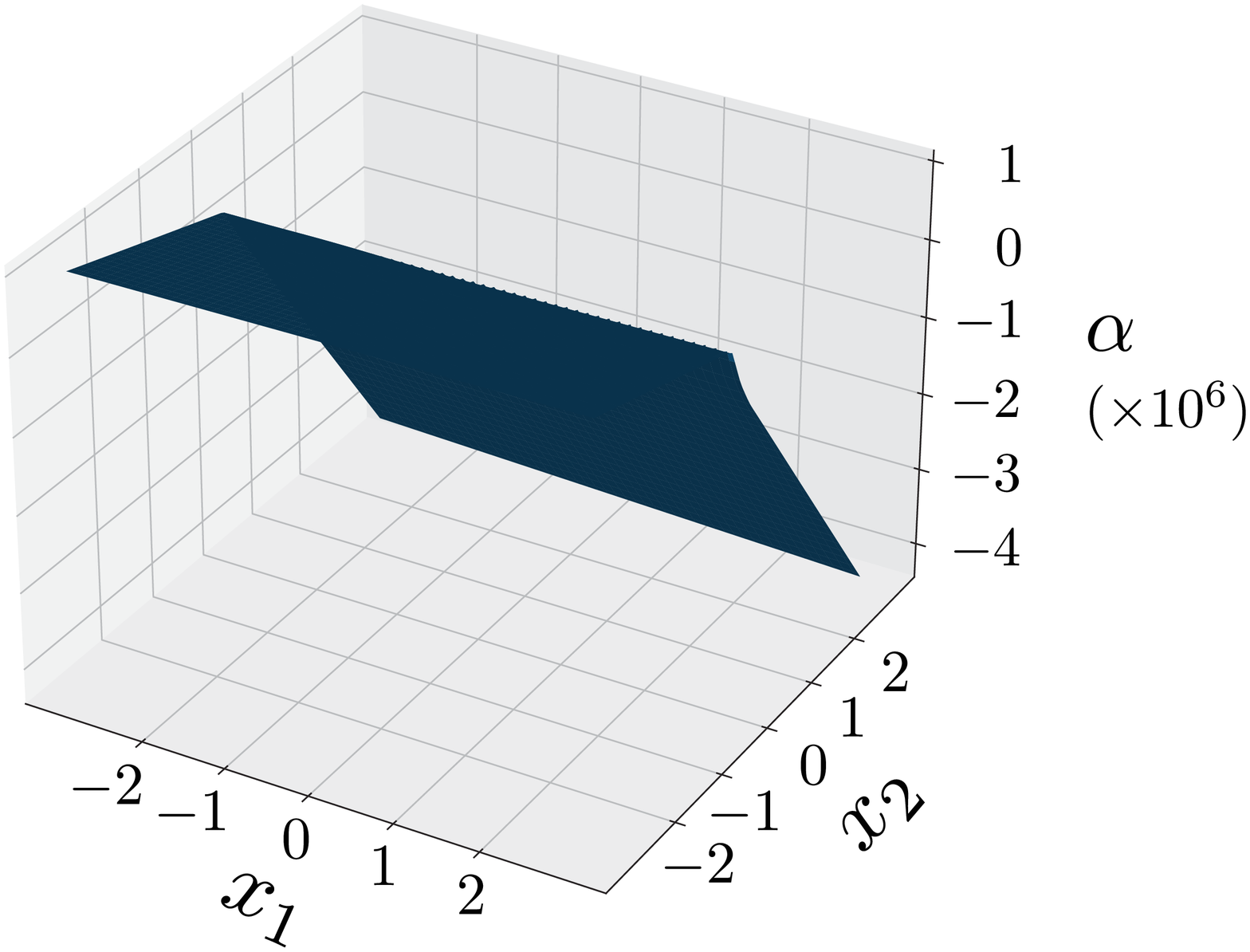}
    \caption{Stabilizing controller $u=\alpha(x)$ learned by Algorithm \ref{alg:CLF} such that the conditions in Theorem \ref{stab:thm} hold}
    \label{im:salpha}
    \end{minipage}
\end{figure}


\section{Conclusion}\label{con:sec}
In this paper, we have developed an algorithm for learning stabilizable dynamics.
We have theoretically guaranteed that the learned dynamics are stabilizable by 
simultaneously learning stabilizing controllers and Lyapunov functions of the closed-loop systems.
It is expected that the proposed algorithm can be applied to various control problems as partly 
illustrated by $H_\infty$-control and optimal control.
Furthermore, the proposed method can be extended to learning dynamics that can be made dissipative with respect to an arbitrary supply rate by control design and to find control barrier functions for safety control, which will be reported in future publication.

\bibliographystyle{IEEEtran} 
\bibliography{CLF_learning}

\end{document}